\documentclass[letterpaper, 10 pt, conference]{ieeeconf}  

\IEEEoverridecommandlockouts                              

\usepackage[english]{babel}
\usepackage{ifpdf}

\usepackage{cite} 
\usepackage{url}

\usepackage{color}
\usepackage{pgf, tikz, pgfplots}
\usetikzlibrary{shapes, arrows, automata}
\usetikzlibrary{calc,hobby,decorations}

\usepackage[cmex10]{amsmath}
\usepackage{amsfonts, amssymb}
\usepackage{mathrsfs}

\usepackage{algorithm} 
\usepackage{algorithmic} 
\usepackage{amsmath} 
\usepackage{xcolor}
\usepackage{enumerate}
\usepackage{multirow}
\usepackage{rotating}
\usepackage{subcaption}
\usepackage{needspace}

\input{mySymbol.sty}

\definecolor{penndarkestblue}{cmyk}{1,0.74,0,0.77}
\definecolor{penndarkerblue}{cmyk}{1,0.74,0,0.70}
\definecolor{pennblue}{cmyk}{0.99,0.66,0,0.57} 
\definecolor{pennlighterblue}{cmyk}{0.98,0.44,0,0.35}
\definecolor{pennlightestblue}{cmyk}{0.38,0.17,0,0.17} 

\definecolor{penndarkestred}{cmyk}{0,1,0.89,0.66}
\definecolor{penndarkerred}{cmyk}{0,1,0.88,0.55}
\definecolor{pennred}{cmyk}{0,1,0.83,0.42} 
\definecolor{pennlighterred}{cmyk}{0,1,0.6,0.24}
\definecolor{pennlightestred}{cmyk}{0,0.43,0.26,0.12} 

\definecolor{penndarkestgreen}{cmyk}{1,0,1,0.68}
\definecolor{penndarkergreen}{cmyk}{1,0,1,0.57}
\definecolor{penngreen}{cmyk}{1,0,1,0.44} 
\definecolor{pennlightergreen}{cmyk}{1,0,1,0.25}
\definecolor{pennlightestgreen}{cmyk}{0.43,0,0.43,0.13}

\definecolor{penndarkestorange}{cmyk}{0,0.65,1,0.49}
\definecolor{penndarkerorange}{cmyk}{0,0.65,1,0.33}
\definecolor{pennorange}{cmyk}{0,0.54,1,0.24} 
\definecolor{pennlighterorange}{cmyk}{0,0.32,1,0.13}
\definecolor{pennlightestorange}{cmyk}{0,0.15,0.46,0.06}

\definecolor{penndarkestpurple}{cmyk}{0,1,0.11,0.86}
\definecolor{penndarkerpurple}{cmyk}{0,1,0.13,0.82}
\definecolor{pennpurple}{cmyk}{0,1,0.11,0.71} 
\definecolor{pennlighterpurple}{cmyk}{0,1,0.05,0.46}
\definecolor{pennlightestpurple}{cmyk}{0,0.35,0.02,0.23}

\definecolor{pennyellow}{cmyk}{0,0.20,1,0.05} 
\definecolor{pennlightgray1}{cmyk}{0,0,0,0.05}
\definecolor{pennlightgray3}{cmyk}{0.01,0.01,0,0.18}
\definecolor{pennmediumgray1}{cmyk}{0.04,0.03,0,0.31}
\definecolor{pennmediumgray4}{cmyk}{0.08,0.06,0,0.54}
\definecolor{penndarkgray2}{cmyk}{0.09,0.07,0,0.71}
\definecolor{penndarkgray4}{cmyk}{0.1,0.1,0,0.92}

\def\SO3{\mathrm{SO(3)}}

\newtheorem{problem}{\hspace{0pt}\bf Problem}

\newtheorem{theorem}{\hspace{0pt}\bf Theorem}

\makeatletter
\let\NAT@parse\undefined
\makeatother
\usepackage{hyperref}

\title{\LARGE \bf
Environment Optimization for Multi-Agent Navigation
}

\author{Zhan Gao and Amanda Prorok
\thanks{Department of Computer Science and Technology, University of Cambridge
        {\tt\small zg292@cam.ac.uk, asp45@cam.ac.uk}. This work was supported by European Research Council (ERC) Project 949940 (gAIa).}%
    }

\begin{document}

\maketitle
\thispagestyle{empty}
\pagestyle{empty}

\begin{abstract}

Traditional approaches to the design of multi-agent navigation algorithms consider the environment as a fixed constraint, despite the obvious influence of spatial constraints on agents' performance. Yet hand-designing improved environment layouts and structures is inefficient and potentially expensive. The goal of this paper is to consider the environment as a decision variable in a system-level optimization problem, where both agent performance and environment cost can be accounted for. We begin by proposing a novel environment optimization problem. We show, through formal proofs, under which conditions the environment can change while guaranteeing completeness (i.e., all agents reach their navigation goals). Our solution leverages a model-free reinforcement learning approach. In order to accommodate a broad range of implementation scenarios, we include both online and offline optimization, and both discrete and continuous environment representations. Numerical results corroborate our theoretical findings and validate our approach.

\end{abstract}


\section{INTRODUCTION}

Multi-agent systems present an attractive solution to spatially distributed tasks, wherein motion planning among moving agents and obstacles is one of the central problems. To date, the primal focus in multi-agent motion planning has been on developing effective, safe, and near-optimal navigation algorithms \cite{silver2005cooperative, van2008reciprocal,van2005prioritized,desaraju2012decentralized, standley2011complete}. These algorithms consider the \textit{agents' environment as a fixed constraint}, where structures and obstacles must be circumnavigated; in this process, mobile agents engage in negotiations with one another for right-of-way, driven by local incentives to minimize individual delays. However, environmental constraints may result in dead-locks, live-locks and prioritization conflicts, even for state-of-the-art algorithms \cite{mani2010search}. These insights highlight the impact of the environment on multi-agent navigation.

Not all environments elicit the same kinds of agent behaviors and individual navigation algorithms are susceptible to environmental artefacts; undesirable environments can lead to irresolution in path planning \cite{ruderman_Uncovering_2018}. To deal with such bottlenecks, spatial structures (e.g., intersections, roundabouts) and markings (e.g., lanes) are developed to facilitate path de-confliction \cite{boudet2021collections} but these concepts are based on legacy mobility paradigms, which ignore inter-agent communication, cooperation, and systems-level optimization. While it is possible to deal with the circumvention of dead-locks and live-locks through hand-designed environment templates, such hand-designing process is inefficient~\cite{cap_Prioritized_2015}.

\begin{figure}%
	\centering
	\begin{subfigure}{0.49\columnwidth}
		\includegraphics[width=1.1\linewidth, height = 0.8\linewidth]{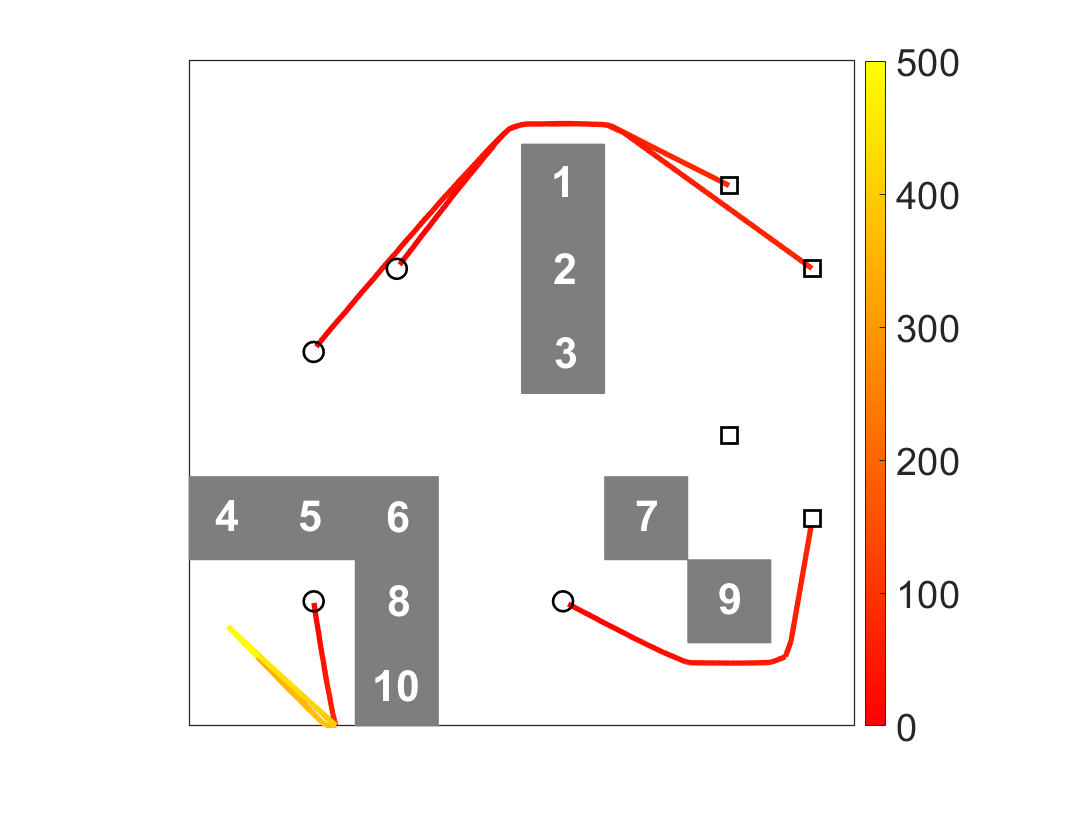}%
		\caption{}%
		\label{subfig_demo_before}%
	\end{subfigure}\hfill\hfill%
	\begin{subfigure}{0.49\columnwidth}
		\includegraphics[width=1.1\linewidth,height = 0.8\linewidth]{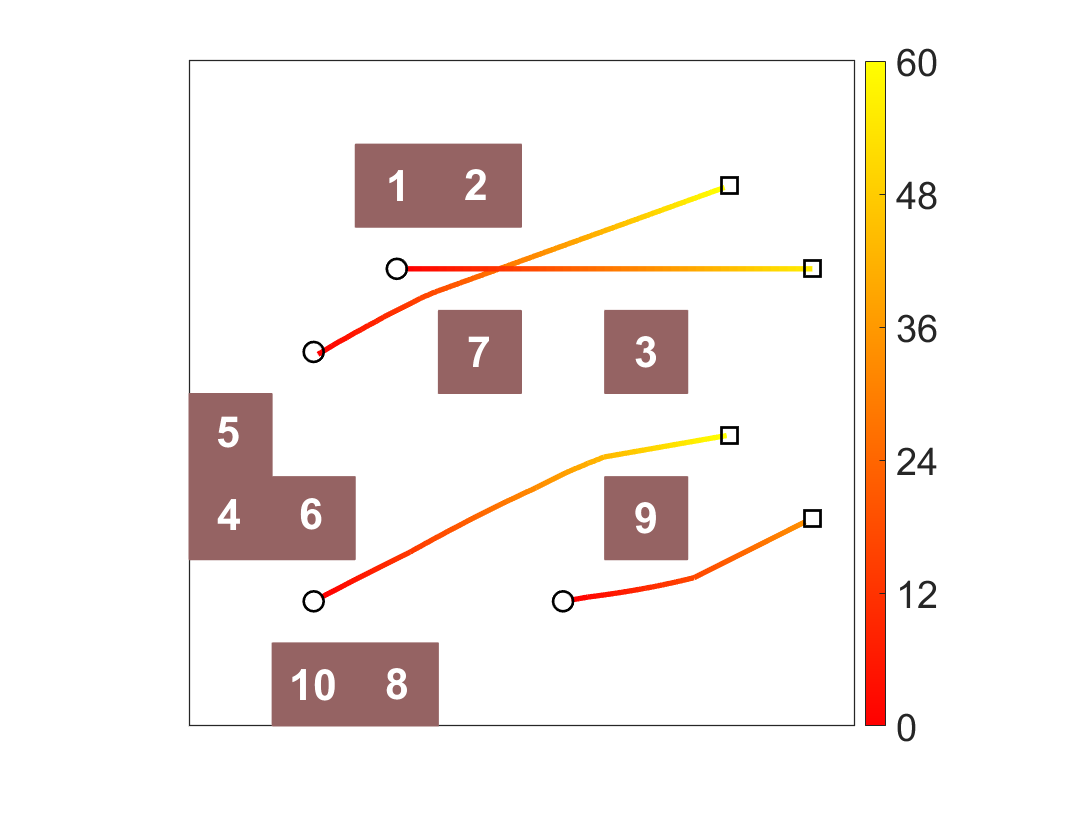}%
		\caption{}%
		\label{subfig_demo_after}%
	\end{subfigure}
	\caption{Example of offline environment optimization. Circles are the initial positions, squares are the destinations and the obstacles are numbered for exposition. The color lines from red to yellow are agent trajectories and the color bar represents the time scale. Obstacle layout and agent trajectories (a) \textit{before} environment optimization and (b) \textit{after} environment optimization.}\label{fig:offlineDemo}\vspace{-6mm}
\end{figure}

Reconfigurable, automated environments are emerging as a new trend~\cite{wang2010new, bier2014robotic, custodio2020flexible}. In tandem with that enabling technology, the goal of this paper is to consider the environment as a \emph{variable} in pursuit of the agents' incentives. We propose the problem of systematically \textit{optimizing an environment to improve the performance of a given multi-agent system}, an example of which is given in Fig. \ref{fig:offlineDemo}. More in detail, our contributions are as follows:
\begin{enumerate}[(i)]
	
	\item We first define the problem of environment optimization. We then develop two solution variants, i.e., the offline environment optimization and the online environment optimization.
	
	\item We analyze the completeness of multi-agent navigation with environment optimization, and identify the conditions under which all agents are guaranteed to reach their navigation goals.
	
	\item We develop a reinforcement learning-based method to solve the environment optimization problem, and integrate two variant information processing architectures (i.e., CNNs, GNNs) as a function of the problem setting. The proposed method is able to generalize beyond training instances, adapts to various optimization objectives, and allows for decentralized implementation. 
	
\end{enumerate}

\noindent \textbf{Related work.} The \textit{role of the environment} on motion planning has been explored by a handful of works~\cite{bennewitz2002finding, jager2001decentralized, vcap2015complete, wu_MultiRobot_2019, gur2021adversarial}. The works in \cite{bennewitz2002finding, jager2001decentralized} 
emphasize the existence of congestion and deadlocks in undesirable environments and develop trajectory planning methods to escape from potential deadlocks. The authors in~\cite{vcap2015complete} define the concept of ``well-formed'' environments, in which the navigation tasks of all agents can be carried out successfully without collisions nor deadlocks. In~\cite{wu_MultiRobot_2019}, Wu et al. show that the shape of the environment leads to distinct \emph{path prospects} for different agents, and that this information can be shared among the agents to optimize and coordinate their mobility. Gur et al. in \cite{gur2021adversarial} generate adversarial environments and account for the latter information to develop resilient navigation algorithms. However, none of these works consider optimizing the environment to improve system-wide navigation performance.

The problem of environment optimization is reminiscent of robot co-design \cite{tanaka2015sdp,tatikonda2004control,tzoumas2018sensing}, wherein sensing, computation, and control are jointly optimized. While this domain of robotics is \textit{robot-centric} (i.e., does not consider the environment as an optimization criteria), there are a few works that, similarly to our approach, use reinforcement learning to co-optimize objectives \cite{lipson2000automatic,hornby2003generative,cheney2018scalable}. A more closely related idea is exploited in \cite{ruderman_Uncovering_2018}, wherein the environment is adversarially optimized to fail navigation tasks of state-of-the-art multi-agent systems. It conducts worst-case analysis to shed light on directions in which more robust systems can be developed. On the contrary, we optimize the environment to facilitate multi-agent navigation tasks.

\section{PROBLEM DEFINITION}

Let $\ccalE$ be a $2$-D environment described by a starting region $\ccalS$, a destination region $\ccalD$ and an obstacle region $\Delta$ without overlap, i.e., $\ccalS \bigcap \ccalD = \ccalS \bigcap \Delta = \ccalD \bigcap \Delta = \emptyset$. Consider a multi-agent system with a set of agents $\ccalA = \{A_i\}_{i=1}^n$ distributed in $\ccalE$. The agent bodies are contained within circles of radius $\{r_i\}_{i=1}^n$. Agents are initialized at positions $\bbS = [\bbs_1,\ldots,\bbs_n]$ in $\ccalS$ and employ a given trajectory planner $\pi_a$ to move towards destinations $\bbD = [\bbd_1,\ldots,\bbd_n]$ in $\ccalD$. We make practical assumptions about $\pi_a$: (i) it is decentralized (can be executed locally with access to local information only), and (ii) in an empty space, it guarantees that all agents reach their destinations while remaining collision-free.\footnote{There are numerous trajectory planners that satisfy these assumptions, such as \cite{van2008reciprocal,van2005prioritized,desaraju2012decentralized}. In this work, we implement a planner based on \cite{van2008reciprocal} (RVO).}

Existing literature focuses on developing novel trajectory planners to improve navigation performance. However, the latter depends not only on the implemented planner but also on the surrounding environment $\ccalE$. A ``well-formed'' environment with an appropriate obstacle region yields good performance for a simple planner, while a ``poorly-formed'' environment may result in a poor performance even for an advanced planner. This insight implies an important role played by the environment in multi-agent navigation, and motivates the problem of environment optimization.
\begin{problem}[Environment Optimization]\label{def:environmenProblem}
	\emph{Given an environment with an initial obstacle region and a multi-agent system with $n$ agents that are required to reach $n$ destinations (in a labeled navigation problem), find a policy that optimizes the obstacle region to improve agents' path efficiency while guaranteeing that they reach their destinations.}
\end{problem}

In Section \ref{sec:Completeness}, we analyze the completeness of multi-agent navigation with environment optimization. In Section \ref{sec:Methodology}, we formulate Problem \ref{def:environmenProblem} mathematically and leverage reinforcement learning for solutions. Lastly, we present numerical simulations to evaluate the approach in Section \ref{sec:Methodology}.

\section{COMPLETENESS ANALYSIS}\label{sec:Completeness}

The multi-agent system may fail to find collision-free (and deadlock-free) trajectories in an environment with an unsatisfactory obstacle region. Environment optimization overcomes this issue by optimizing the layout and guaranteeing  navigation success under a mild condition. We propose two variants: the \emph{offline} and the \emph{online} environment optimization.  

\subsection{Offline Environment Optimization}

Offline environment optimization optimizes the obstacle region $\Delta$ based on the starting region $\ccalS$, the destination region $\ccalD$ and completes the optimization before the agents start to move. The optimized environment remains fixed during agent movement. We introduce some necessary notation for theoretical analysis. Let $\partial \ccalS$, $\partial \ccalD$, $\partial\Delta$ be the boundaries of the regions $\ccalS$, $\ccalD$, $\Delta$ and $B(\bbs,r)$ be a closed disk centered at position $\bbs$ with a radius $r$. We can represent each agent $A_i$ as the disk $B(\bbs_i, r_i)$ where $\bbs_i$ is the central position of the agent. Define $d(\bbs_i, \bbs_j)$ as the closest distance between agents $A_i$, $A_j$, i.e., $d(\bbs_i, \bbs_j) = \min \|\bbz_i - \bbz_j\|_2$ for any $\bbz_i \in B(\bbs_i, r_i)$ and $\bbz_j \in B(\bbs_j, r_j)$, and $d(\bbs_i, \partial \ccalS)$ the closest distance between agent $A_i$ and the region boundary $\partial \ccalS$, i.e., $d(\bbs_i, \partial \ccalS) = \min \|\bbz_i - \bbz_\ccalS\|_2$ for any $\bbz_i \in B(\bbs_i, r_i)$ and $\bbz_\ccalS \in \partial \ccalS$. Similar definitions apply for $\partial\ccalD$ and $\partial\Delta$. 

Our analysis assumes the following setup. The initial positions $\{\bbs_i\}_{i=1}^n$ in $\ccalS$ and the destinations $\{\bbd_i\}_{i=1}^n$ in $\ccalD$ are distributed in a way such that the distance between either two agents or the agent and the region boundary is larger than the maximal agent size, i.e., $d(\bbs_i,\bbs_j) \ge 2 \hat{r}$, $d(\bbs_i,\partial \ccalS) \ge 2 \hat{r}$ and $d(\bbd_i,\bbd_j) \ge 2 \hat{r}$, $d(\bbd_i,\partial \ccalD) \ge 2 \hat{r}$ for $i,j=1,\ldots,n$ with $\hat{r} = \max_{i=1,\ldots,n} r_i$ the maximal agent radius. The obstacle region $\Delta$ can be controlled / changed \textbf{but its area $|\Delta|$ is constant}, i.e., $|\Delta| = |\widetilde{\Delta}|$ where $\widetilde{\Delta}$ is the changed obstacle region. That is, the obstacle region (or any part of it) cannot be removed from the environment. Each agent employs a given trajectory planner $\pi_a$ and any trajectory generated for an agent will be followed precisely. Denote by $\bbp(t): [0,\infty) \to \mathbb{R}^2$ the trajectory representing the central position movement of an agent. The trajectory $\bbp_i$ of agent $A_i$ is collision-free w.r.t. $\Delta$ if $d\big(\bbp_i(t), \partial \Delta \big) \ge r_i$ for all $t\ge 0$. The trajectories $\bbp_i$, $\bbp_j$ of two agents $A_i$, $A_j$ are collision-free if $\|\bbp_i(t) - \bbp_j(t)\|_2 \ge r_i + r_j$ for all $t \ge 0$. With these preliminaries, we show the completeness in the following.

\begin{theorem}\label{thm:offlineCompleteness}
	\emph{Consider the multi-agent system in the environment $\ccalE$ with starting, destination and obstacle regions $\ccalS$, $\ccalD$ and $\Delta$. Let $d_{\max}$ be the maximal distance between $\ccalS$ and $\ccalD$, i.e., $d_{\max} = \max_{\bbz_\ccalS,\bbz_\ccalD} \| \bbz_\ccalS - \bbz_\ccalD \|_2$ for any $\bbz_\ccalS \in \ccalS,\bbz_\ccalD \in \ccalD$. 
	Then, if the environment $\ccalE$ satisfies
	\begin{align}\label{eq:offlineCompletenessCondition}
		|\ccalE \setminus (\Delta \cup \ccalS \cup \ccalD)| \ge 2 n d_{\max} \hat{r}
	\end{align}
	where $\hat{r} = \max_{i=1,\ldots,n} r_i$ is the maximal radius of $n$ agents and $|\cdot|$ represents the region area, the environment optimization guarantees that the navigation tasks of all agents will be carried out successfully without collision.} 
\end{theorem}

\begin{proof}
	We prove the theorem as follows. First, we optimize $\Delta$ such that the environment is ``well-formed'', i.e., any initial position in $\ccalS$ and destination in $\ccalD$ can be connected by a collision-free path. Then, we show the optimized environment guarantees the success of all navigation tasks. 
	
	\smallskip
	\noindent \textbf{Obstacle region optimization.} We first optimize $\Delta$ based on $\ccalS$, $\ccalD$ to make the environment "well-formed". To do so, we handle $\ccalS$, $\ccalD$ and the remaining space $\ccalE \setminus (\ccalS \bigcup \ccalD)$ separately. 
	
	\emph{(i)} The initial positions $\{\bbs_i\}_{i=1}^n$ in $\ccalS$ are distributed such that $d(\bbs_i,\bbs_j) \ge 2 \hat{r}$ and $d(\bbs_i,\partial \ccalS) \ge 2 \hat{r}$. Thus, for any $\bbs_i$, there exists a boundary point $\partial \bbs_i \in \partial \ccalS$ and a path $\bbp_{\bbs_i}^{\partial \bbs}$ connecting $\bbs_i$ and $\partial \bbs_i$ that is collision-free w.r.t. the other initial positions. A similar result applies to the destinations $\{\bbd_i\}_{i=1}^n$ in $\ccalD$, i.e., for any $\bbd_i$, there exists a boundary point $\partial \bbd_i \in \partial \ccalD$ and a path $\bbp_{\partial \bbd_i}^{\bbd_i}$ connecting $\partial \bbd_i$ and $\bbd_i$ that is collision-free w.r.t. the other destinations.
	
	\emph{(ii)} Consider $\partial \bbs_i$ and $\partial \bbd_i$ for agent $A_i$. The shortest path $\bbp_{\partial \bbs_i}^{\partial \bbd_i}$ that connects them is the straight path, the area of which is bounded as $|\bbp_{\partial \bbs_i}^{\partial \bbd_i}| \le  2 \| \partial \bbs_i - \partial \bbd_i\|_2\hat{r}\le 2 d_{\max}\hat{r}$ because $d_{\max}$ is the maximal distance between $\ccalS$ and $\ccalD$. From $|\ccalE \setminus (\Delta \cup \ccalS \cup \ccalD)| \ge 2 n d_{\max} \hat{r}$ in \eqref{eq:offlineCompletenessCondition}, the area of the obstacle-free space in $\ccalE \setminus (\ccalS \bigcup \ccalD)$ is larger than $2 n d_{\max} \hat{r}$. Thus, we can always optimize $\Delta$ to $\Delta^*$ such that the path $\bbp_{\partial \bbs_i}^{\partial \bbd_i}$ is obstacle-free. If $\bbp_{\partial \bbs_i}^{\partial \bbd_i}$ dose not overlap with $\ccalS$ and $\ccalD$, we can connect $\partial \bbs_i$ and $\partial \bbd_i$ with $\bbp_{\partial \bbs_i}^{\partial \bbd_i}$ directly. If $\bbp_{\partial \bbs_i}^{\partial \bbd_i}$ passes through $\ccalS$ $K$ times, e.g., for $K$ distinct paths, let $\bbs_{i,e}^{(k)}$ and $\bbs_{i,l}^{(k)}$ be the entering and leaving positions of $\bbp_{\partial \bbs_i}^{\partial \bbd_i}$ on $\ccalS$ at $k$th pass for $k=1,...,K$ with $\bbs_{i,l}^{(0)} = \partial \bbs_i$ the initial leaving position. First, we can connect $\bbs_{i,l}^{(k-1)}$ and $\bbs_{i,e}^{(k)}$ by $\bbp_{\partial \bbs_i}^{\partial \bbd_i}$ because $\bbp_{\partial \bbs_i}^{\partial \bbd_i}$ is obstacle-free. Then, there exists a collision-free path $\bbp_i^{(k)}$ inside $\ccalS$ that connects $\bbs_{i,e}^{(k)}$ and $\bbs_{i,l}^{(k)}$ as described in (i). The same result applies so that $\bbp_{\partial \bbs_i}^{\partial \bbd_i}$ passes through $\ccalD$. Therefore, we can connect $\partial \bbs_i$ and $\partial \bbd_i$ with $\bbp_{\partial \bbs_i}^{\partial \bbd_i}$ and $\{\bbp_i^{(k)}\}_{k=1}^K$. 
	
	By concatenating $\bbp_{\bbs_i}^{\partial \bbs_i}$, $\bbp_{\partial \bbs_i}^{\partial \bbd_i}$, $\{\bbp_i^{(k)}\}_{k=1}^K$ and $\bbp_{\partial \bbd_i}^{\bbd_i}$, we can establish a path $\bbp_{\bbs_i}^{\bbd_i}$ connecting $\bbs_i$ to $\bbd_i$ that is collision-free w.r.t. the other initial positions, destinations and the optimized obstacle region $\Delta^*$ for $i=1,...,n$, i.e., the optimized environment is ``well-formed''.
	
	\smallskip
    \noindent \textbf{Completeness.} With the fact that the optimized environment is ``well-formed'', by Theorem 4 in \cite{vcap2015complete}, we complete the proof that all navigation tasks can be carried out successfully under both centralized as well as decentralized (e.g., priority-based) planners. 
\end{proof}

Theorem \ref{thm:offlineCompleteness} states offline environment optimization guarantees the success of all navigation tasks under a mild condition. It does not require any initial ``well-formed'' environment but only a small obstacle-free area [cf. \eqref{eq:offlineCompletenessCondition}], which is commonly satisfied in real-world scenarios. The offline environment optimization depends on the given starting region $\ccalS$ and destination region $\ccalD$, and completes optimizing the obstacle region \emph{before} the agents start to move. This requires a computational overhead before each new navigation task. Moreover, we are interested in generalizing the problem s.t. $\ccalS$ and $\ccalD$ are allowed to be\textit{ time-varying} during deployments.

\subsection{Online Environment Optimization} 

We propose an online environment optimization variant to overcome the above-mentioned issues by changing the obstacle region during agent movement. Different from its offline counterpart, it is interleaved with the deployment online, i.e., it changes the obstacle region based on instantaneous system states, and stays active until the end of the navigation. Specifically, define the starting region as the union of initial positions $\ccalS = \bigcup_{i=1,...,n}B(\bbs_i, r_i)$ and the destination region as that of destinations $\ccalD = \bigcup_{i=1,...,n}B(\bbd_i, r_i)$ such that $\ccalS \bigcap \ccalD = \emptyset$. The decentralized trajectory planner $\pi_a$ is given (i.e., each agent plans locally according to $\pi_a$). For scenarios with an empty obstacle region $\Delta = \emptyset$, the environment $\ccalE$ is ``well-formed'' such that all navigation tasks will be carried out successfully without collision. We denote these trajectories by $\{\bbp_i\}_{i=1}^n$. This is a reasonable assumption because an environment without obstacles is the best scenario for multi-agent navigation. For scenarios with obstacles, i.e., $\Delta \neq \emptyset$, these navigation tasks may fail and the online environment optimization handles the latter by changing $\Delta$ during the navigation procedure. Since $\Delta$ now changes continuously, we define the \emph{capacity} of the online environment optimization as the maximal changing rate of the obstacle area $\dot{\Delta}$, i.e., the maximal obstacle area that can be changed per time step. The following theorem formally shows the completeness.  
\begin{theorem}\label{thm:onlineCompleteness}
	\emph{Consider the multi-agent system with $n$ agents $\{A_i\}_{i=1}^n$. For the environment $\ccalE$ without obstacles, i.e., $\Delta = \emptyset$, let $\{\bbp_i(t)\}_{i=1}^n$ be $n$ collision-free trajectories of $\{A_i\}_{i=1}^n$ generated by the given trajectory planner $\pi_a$ and $\{\bbv_i(t)\}_{i=1}^n$ be the respective velocities along these trajectories. For the environment $\ccalE$ with an obstacle region $\Delta \subset \ccalE \setminus (\ccalS \cup \ccalD)$ and $\ccalW \setminus (\Delta \cup \ccalS \cup \ccalD) \ne \emptyset$, if the capacity of the online environment optimization satisfies
	\begin{align}\label{eq:onlineCompletenessCondition}
		\dot{\Delta} \ge 2 n \hat{r} \| \hat{\bbv} \|_2
	\end{align}
	where $\|\hat{\bbv}\|_2 = \max_{t \in [0,T]}\max_{i=1,...,n}\|\bbv_i(t)\|_2$ is the maximal norm of the velocities and $\hat{r} = \max_{i=1,\ldots,n} r_i$ is the maximal agent radius, the navigation tasks of all agents will be carried out successfully without collision.} 
\end{theorem}
\begin{proof}
	We prove the theorem as follows. First, we slice the navigation procedure into $H$ time frames. Then, we optimize the obstacle region based on the agent positions in each frame, and show the completeness of individual frames. Lastly, we show the completeness of the entire multi-agent navigation solution by concatenating individual frames, and complete the proof by considering the whole process when the number of frames tends to infinity, i.e., $H \to \infty$. 	
	
	\noindent \textbf{Time slicing.} Let $T$ be the maximal operation time of trajectories $\{\bbp_i\}_{i=1}^n$ and $\{[(h-1)T/H, hT/H]\}_{h=1}^{H}$ the time frames. This yields intermediate positions $\big\{\bbp_i(hT/H)\big\}_{h=0}^{H}$ with $\bbp_i(0)=\bbs_i$ and $\bbp_i(T)=\bbd_i$ for $i=1,...,n$. We can re-formulate the navigation task into $H$ sub-navigation tasks, where the $h$th sub-task of agent $A_i$ is from $\bbp_i((h-1)T/H)$ to $\bbp_i(hT/H)$ and the operation time of the sub-navigation task is $\delta t = T/H$. In each frame, we first change the obstacle region based on the corresponding sub-navigation task and then navigate the agents until the next frame. 
	
	\noindent \textbf{Obstacle region optimization.} We consider each sub-navigation task separately and start from the first one. Assume the obstacle region $\Delta$ satisfies \vspace{-1mm}
	\begin{align}\label{eq:proofTheorem21}
		|\ccalW \setminus (\Delta \cup \ccalS \cup \ccalD)| > 2 n \hat{r} \|\hat{\bbv}\|_2 \delta t.
	\end{align}
	For the $1$st sub-navigation task, the starting region is $\ccalS^{(1)} = \bigcup_{i=1,\ldots,n} B(\bbp_i(0), r_i) = \ccalS$ and the destination region is $\ccalD^{(1)} = \bigcup_{i=1,\ldots,n} B(\bbp_i(T/H), r_i)$.  We optimize $\Delta$ based on $\ccalS^{(1)}$, $\ccalD^{(1)}$ and show the completeness of the $1$st sub-navigation task, which consists of two steps. First, we change $\Delta$ to $\widetilde{\Delta}$ such that $|\Delta| = |\widetilde{\Delta}|$ and $\widetilde{\Delta} \subset \ccalW \setminus (\ccalS^{(1)} \cup \ccalD^{(1)})$. This can be completed as follows. From the condition $\Delta \subset \ccalE \setminus (\ccalS \cup \ccalD)$ and $\ccalS^{(1)} = \ccalS$, there is no overlap between $\Delta$ and $\ccalS^{(1)}$. For any overlap region $\Delta \bigcap \ccalD^{(1)}$, we can change it to the obstacle-free region in $\ccalD$ because $\Delta \subset \ccalE \setminus (\ccalS \cup \ccalD)$ and $|\ccalD^{(1)}| = |\ccalD|$, and keep the other region in $\Delta$ unchanged. The resulting $\widetilde{\Delta}$ satisfies $|\widetilde{\Delta}| = |\Delta|$ and $\widetilde{\Delta} \subset \ccalW \setminus (\ccalS^{(1)} \cup \ccalD^{(1)})$. The changed area from $\Delta$ to $\widetilde{\Delta}$ is bounded by $|\ccalD^{(1)}|\le n \pi \hat{r}^2$. Second, we change $\widetilde{\Delta}$ to $\Delta^{(1)}$ such that the environment is ``well-formed'' w.r.t. the $1$st sub-navigation task. The initial position $\bbp_i(0)$ and the destination $\bbp_i(H/T)$ can be connected by a path $\bbp_i^{(1)}$ that follows the trajectory $\bbp_i$. Since $\|\hat{\bbv}\|_2$ is the maximal speed and $\delta t$ is the operation time, the area of $\bbp_i^{(1)}$ is bounded by $2 \hat{r} \| \hat{\bbv} \|_2 \delta t$. Since this holds for all $i=1,...,n$, we have $\sum_{i=1}^n |\bbp_i^{(1)}| \le 2 n \hat{r} \| \hat{\bbv} \|_2 \delta t$. From \eqref{eq:proofTheorem21}, $|\ccalS^{(1)}| = |\ccalS|$, $|\ccalD^{(1)}| = |\ccalD|$ and $|\widetilde{\Delta}| = |\Delta|$, we have \vspace{-1mm}
	\begin{align}\label{eq:proofTheorem23}
		|\ccalW \setminus (\widetilde{\Delta} \cup \ccalS^{(1)} \cup \ccalD^{(1)})| > 2 n \hat{r} \|\hat{\bbv}\|_2 \delta t.
	\end{align}		
	This implies the obstacle-free area in $\ccalE$ is larger than the area of $n$ paths $\{\bbp_i^{(1)}\}_{i=1}^n$. Following the proof of Theorem \ref{thm:offlineCompleteness}, we can optimize $\widetilde{\Delta}$ to $\Delta^{(1)}$ to guarantee the success of the $1$st sub-navigation task. The changed area from $\widetilde{\Delta}$ to $\Delta^{(1)}$ is bounded by $2 n \hat{r} \|\hat{\bbv}\|_2 \delta t - n \pi \hat{r}^2$ because the initial positions $\{\bbp_i(0)\}_{i=1}^n$ and destinations $\{\bbp_i(T/H)\}_{i=1}^n$ in $\{\bbp_i^{(1)}\}_{i=1}^n$ are collision-free from the first step, which dose not require any further change of the obstacle region. The total changed area from $\Delta$ to $\Delta^{(1)}$ can be bounded as \vspace{-1mm}
	\begin{align}\label{eq:proofTheorem24}
		\frac{\big|(\Delta \bigcup \Delta^{(1)}) \setminus (\Delta \bigcap \Delta^{(1)})\big|}{2} \le 2 n \hat{r} \|\hat{\bbv}\|_2 \delta t.
	\end{align}
    
    From \eqref{eq:proofTheorem23}, $|\Delta^{(1)}| = |\widetilde{\Delta}|$ and $\Delta^{(1)} \subset \ccalW \setminus (\ccalS^{(1)} \cup \ccalD^{(1)})$, the optimized $\Delta^{(1)}$ satisfies $|\ccalW \setminus (\Delta^{(1)} \cup \ccalS^{(1)} \cup \ccalD^{(1)})| \ge 2 n \hat{r} \|\hat{\bbv}\|_2 \delta t$, which recovers the assumption in \eqref{eq:proofTheorem21}. Therefore, we can repeat the above process iteratively and guarantee the success of $H$ sub-navigation tasks. The entire navigation task is guaranteed success by concatenating these sub-tasks. 
	
	\noindent \textbf{Completeness.} When $H \to \infty$, we have $\delta t \to 0$. Since the environment optimization time is same as the agent operation time at each sub-navigation task, the obstacle region and the agents can be considered taking actions simultaneously when $\delta t \to 0$. The initial environment condition in \eqref{eq:proofTheorem21} becomes \vspace{-1mm}
	\begin{align}\label{eq:proofTheorem27}
		\lim_{\delta t \to 0}| \ccalW \setminus (\Delta \cup \ccalS \cup \ccalD)| > 2 n \hat{r} \|\hat{\bbv}\|_2 \delta t \to 0.
	\end{align}	 
	which is satisfied from the condition $\ccalW \setminus (\Delta \cup \ccalS \cup \ccalD) \ne \emptyset$. The changed area of the obstacle region in \eqref{eq:proofTheorem24} becomes \vspace{-1mm}
	\begin{align}\label{eq:proofTheorem28}
		\!\lim_{\delta t \to 0}\!\frac{\big|(\Delta^{(h)} \!\bigcup\! \Delta^{(h+1)}) \!\setminus\! (\Delta^{(h)} \!\bigcap\! \Delta^{(h+1)})\big|}{2 \delta t} \!\le\! 2 n \hat{r} \|\hat{\bbv}\|_2.
	\end{align}
    That is, if the capacity of the online environment optimization is higher than $2 n \hat{r} \|\hat{\bbv}\|_2$, i.e., $\dot{\Delta} \ge 2 n \hat{r} \| \hat{\bbv} \|_2$, the navigation task can be carried out successfully without collision, which completes the proof.	
\end{proof}

Theorem \ref{thm:onlineCompleteness} states online environment optimization guarantees the success of all navigation tasks as well as its offline counterpart. The result is established under a mild condition on the changing rate of the obstacle region [cf. \eqref{eq:onlineCompletenessCondition}] rather than the initial obstacle-free area [cf. \eqref{eq:offlineCompletenessCondition}]. This stems from the fact that online environment optimization changes the obstacle region \textit{while the agents are moving}, which improves navigation performance only if timely actions can be taken based on instantaneous system states. The completeness analysis in Theorems \ref{thm:offlineCompleteness}-\ref{thm:onlineCompleteness} provide theoretical guarantees, demonstrating the applicability of the proposed idea.\footnote{The theoretical analysis in Section \ref{sec:Completeness} is general, i.e., the obstacle region can be of any shape and move continuously.}

\section{METHODOLOGY}\label{sec:Methodology}

In this section, we formulate Problem \ref{def:environmenProblem} mathematically and solve the latter by leveraging reinforcement learning. Specifically, consider the obstacle region $\Delta$ comprising $m$ obstacles $\ccalO = \{O_1,...,O_m\}$, which can be of any shape and are located at positions $\bbO = [\bbo_1,...,\bbo_m]$. These obstacles can change positions to improve system performance (i.e., efficiency of agent navigation and cost of obstacle motion). Denote by $\bbX_o = [\bbx_{o1},...,\bbx_{om}]$ the obstacle states, $\bbU_o = [\bbu_{o1},...,\bbu_{om}]$ the obstacle actions, $\bbX_a = [\bbx_{a1},...,\bbx_{an}]$ the agent states and $\bbU_a = [\bbu_{a1},...,\bbu_{an}]$ the agent actions. Let $\pi_o(\bbU_o | \bbX_o, \bbX_a)$ be a policy that controls the obstacles, which is a distribution of $\bbU_o$ conditioned on $\bbX_o$, $\bbX_a$. The objective function $f(\bbS, \bbD, \pi_a | \bbO, \pi_o)$ measures the multi-agent navigation performance, while the cost function $\{g_i(\bbO, \pi_o | \bbS, \bbD, \pi_a)\}_{i=1}^m$ indicates environment restrictions on the obstacles. The goal is to find a policy $\pi_o$ that maximizes the system performance $f(\bbS, \bbD, \pi_a | \bbO, \pi_o)$ regularized by the obstacle costs $\{g_i(\bbO, \pi_o | \bbS, \bbD, \pi_a)\}_{i=1}^m$. With the obstacle action space $\ccalU_o$, we formulate Problem \ref{def:environmenProblem} as \vspace{-2mm}
\begin{align}\label{eq:environmentProblemMath}
	&\argmax_{\pi_o}~~ f(\bbS, \bbD, \pi_a | \bbO, \pi_o) - \sum_{i=1}^m \beta_i g_i(\bbO, \pi_o | \bbS, \bbD, \pi_a)\nonumber\\
	&~~~~\text{s.t.}~~~~~~ \pi_{o}(\bbU_{o}|\bbX_o,\bbX_a) \in \ccalU_{o} 
\end{align}
Solving \eqref{eq:environmentProblemMath} is difficult and there are mainly three challenges: 

\begin{enumerate}[(i)]
	
	\item The objective function $f(\bbS, \bbD, \pi_a | \bbO, \pi_o)$ and the cost functions $\{g_i(\bbO, \pi_o | \bbS, \bbD, \pi_a)\}_{i=1}^m$ may be complex, inaccurate or unknown, precluding the application of conventional model-based methods and leading to poor performance of heuristic methods.
	
	\item The policy $\pi_o(\bbU_o | \bbX_o, \bbX_a)$ is an arbitrary mapping from the state space to the action space, which can take any function form and is infinitely dimensional. 
	
	\item The obstacle actions can be discrete or continuous and the action space $\ccalU_o$ can be non-convex, resulting in complicated constraints on the feasible solution.
	
\end{enumerate}

Due to the aforementioned challenges, we develop a model-free policy gradient-based reinforcement learning (RL) method to solve this problem.

\subsection{Reinforcement Learning}

We re-formulate problem \eqref{eq:environmentProblemMath} within an RL setting, and begin by defining a Markov decision process. At each time $t$, the obstacles $\ccalO$ and agents $\ccalA$ observe the states $\bbX_o^{(t)}$, $\bbX_a^{(t)}$ and take the actions $\bbU_o^{(t)}$, $\bbU_a^{(t)}$ with the policy $\pi_o$ and trajectory planner $\pi_a$. The actions $\bbU_o^{(t)}$, $\bbU_a^{(t)}$ amend the states $\bbX_o^{(t)}$, $\bbX_a^{(t)}$ based on the transition probability function $P(\bbX_o^{(t+1)},\bbX_a^{(t+1)}|\bbX_o^{(t)},\bbX_a^{(t)},\bbU_o^{(t)},\bbU_a^{(t)})$, which is a distribution of the states conditioned on the previous states and actions. Let $r_{ai}(\bbX_o^{(t)},\bbX_a^{(t)})$ be the reward function of agent $A_i$, which represents instantaneous navigation performance at time $t$. The reward function of obstacle $O_j$ comprises two components: (i) the multi-agent system reward and (ii) the obstacle individual reward, i.e., \vspace{-2mm} 
\begin{align}\label{eq:obstacleRewardFunction}
	&r_{oj} = \frac{1}{n}\sum_{i=1}^n r_{ai} + \beta_j r_{j, \mathrm{local}},~\forall~j=1,...,m
\end{align}
where $r_{ai}$, $r_{j, \mathrm{local}}$ are concise notations of $r_{ai}(\bbX_o,\bbX_a)$, $r_{j, \mathrm{local}} (\bbX_o,\bbX_a)$ and $\beta_j$ is the regularization parameter. 
The first term is the team reward that represents the multi-agent system performance and is shared across all obstacles. The second term is the individual reward that corresponds to environment restrictions on obstacle $O_j$ and may be different among obstacles, such as limitations of moving velocity and distance. 
This reward definition is novel that differs from common RL scenarios. In particular, the obstacle reward is a combination of a global reward with a local reward. The former is the main goal of all obstacles, while the latter is the imposed penalty of a single obstacle. The regularization parameters $\{\beta_j\}_{j=1}^m \in [0, 1]^m$ weigh the environment restriction on the navigation performance. The total reward of the obstacles is defined as $r_{o}= \sum_{j=1}^m r_{oj}$. Let $\gamma \in [0,1]$ be the discount factor accounting for the future rewards and the expected discounted reward can be represented as \vspace{-2mm}
\begin{align}\label{eq:expectedCost}
	&R(\bbO, \bbS, \bbD | \pi_o, \pi_a) = \mathbb{E} \Big[\sum_{t = 0}^\infty \gamma^{t} r_{o}^{(t)}\Big] \\
	& = \mathbb{E}\Big[\sum_{t=0}^\infty \gamma^t \sum_{j=1}^m\!\sum_{i=1}^n\! \frac{r_{ai}^{(t)}}{n}\Big] \!+\! \sum_{j=1}^m \beta_j \mathbb{E}\Big[\sum_{t=0}^\infty \gamma^t r_{j, \rm local}^{(t)}\Big] \nonumber
\end{align}
where $\bbO$, $\bbS$ and $\bbD$ are initial positions and destinations that determine initial states $\bbX_o^{(0)}$ and $\bbX_a^{(0)}$, and the expectation $\mathbb{E}[\cdot]$ is w.r.t. the action policy and state transition probability. By comparing \eqref{eq:expectedCost} with \eqref{eq:environmentProblemMath}, we see equivalent representations of the objective and cost functions in the RL domain. The policy $\pi_o$ is modeled by an information processing architecture $\bbPhi(\bbX_o, \bbX_a, \bbtheta)$ with parameters $\bbtheta$, which are updated through gradient ascent using policy gradients~\cite{sutton1999policy}. The policy distribution is chosen w.r.t. the action space $\ccalU_o$. 

\subsection{Information Processing Architecture}

The above framework covers a variety of environment optimization scenarios and different information processing architectures can be integrated to solve different variants. We illustrate this fact by analyzing two representative scenarios: (i) offline optimization in discrete environments and (ii) online optimization in continuous environments, for which we use convolutional neural networks (CNNs) and graph neural networks (GNNs), respectively.

\noindent \textbf{CNNs for offline discrete settings.} In this setting, we first optimize the obstacles' positions and then navigate the agents. Since computations are performed offline, we collect the states of all obstacles $\bbX_o$ and agents $\bbX_a$ apriori, which allows for centralized information processing solutions (e.g., CNNs). CNNs leverage convolutional filters to extract features from image signals and have found wide applications in computer vision \cite{browne2003convolutional, kumra2017robotic, gu2018recent}. In the discrete environment, the system states can be represented by matrices and the latter are equivalent to image signals. This motivates to consider CNNs for policy parameterization.

\noindent \textbf{GNNs for online continuous settings.} For online optimization in continuous environments, obstacle positions change while agents move, i.e., the obstacles take actions instantaneously. In large-scale systems with real-time communication and computation constraints, obstacles may not be able to collect the states of all other obstacles / agents and centralized solutions may not be applicable. 
This requires a \emph{decentralized architecture} that can be implemented with local neighborhood information. GNNs are inherently decentralizable and are, thus, a suitable candidate. 

GNNs are layered architectures that leverage a message passing mechanism to extract features from graph signals \cite{scarselli2008graph, velivckovic2018graph, gao2021stochastic}. At each layer $\ell$, let $\bbX_{\ell-1}$ be the input signal. The output signal is generated with the message aggregation function $\ccalF_{(\ell-1)m}$ and the feature update function $\ccalF_{(\ell-1) u}$ as
\begin{align}
	[\bbX_{\ell}]_i \!=\! \ccalF_{\ell u}\Big( [\bbX_{\ell\!-\!1}]_i,\!\! \sum_{j \in \ccalN_i}\!\! \ccalF_{\ell m}\big( [\bbX_{\ell\!-\!1}]_i, [\bbX_{\ell\!-\!1}]_j, [\bbE]_{ij} \big)\! \Big)\nonumber
\end{align}
where $[\bbX_\ell]_i$ is the signal value at node $i$, $\ccalN_i$ are the neighboring nodes within the communication radius, $\bbE$ is the adjacency matrix, and $\ccalF_{\ell m}$, $\ccalF_{\ell u}$ have learnable parameters $\bbtheta_{\ell m}$, $\bbtheta_{\ell u}$. The output signal is computed with local neighborhood information only, and each node can make decisions based on its own output value; 
hence, allowing for decentralized implementation \cite{tolstaya2020learning, kortvelesy2021modgnn, gao2022wide, li2020graph}.

\section{EXPERIMENTS}

\begin{figure}%
	\centering
	\begin{subfigure}{0.48\columnwidth}
		\includegraphics[width=1.05\linewidth, height = 0.8\linewidth]{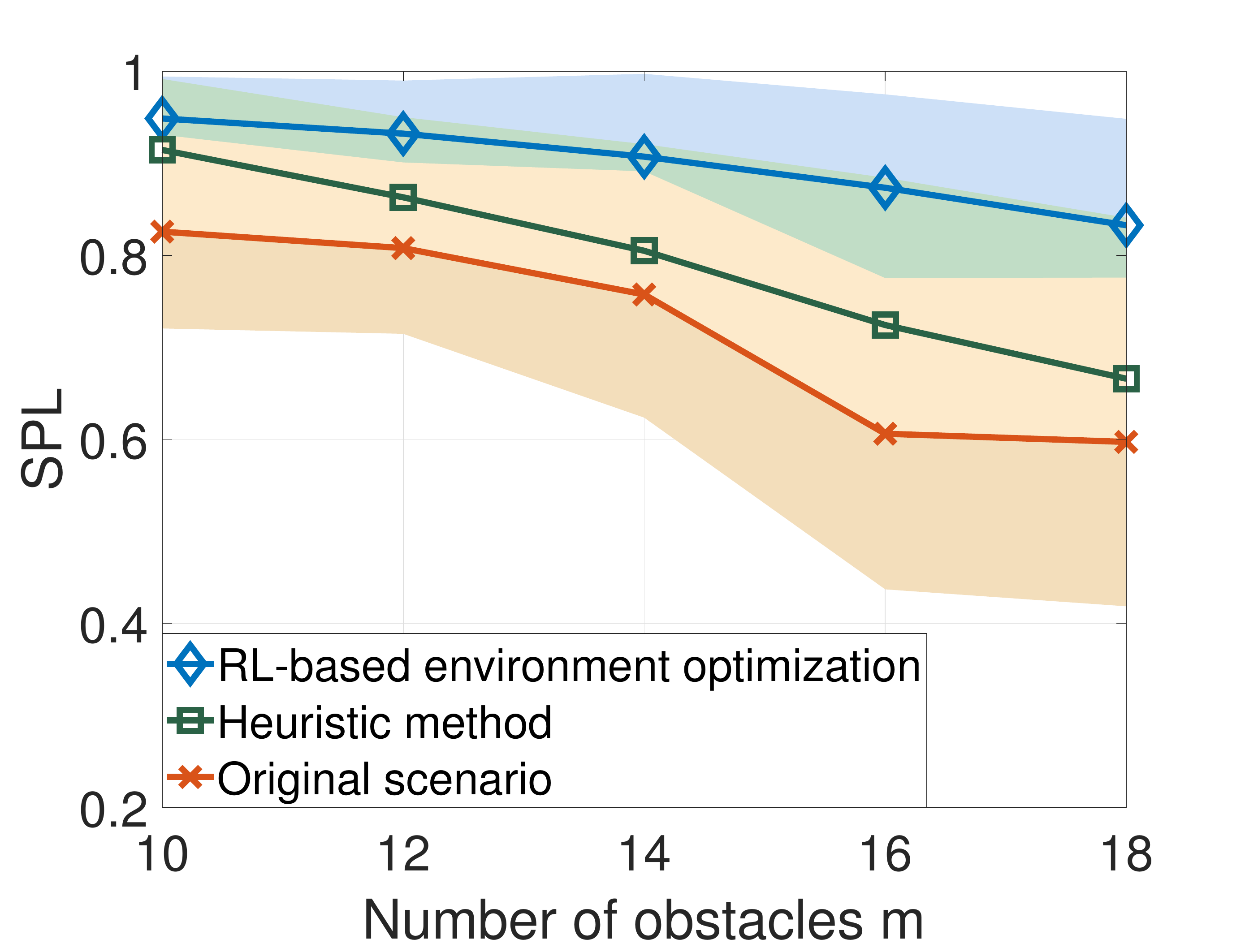}%
		\caption{}%
		\label{subfig_offlineSPL}%
	\end{subfigure}\hfill\hfill%
	\begin{subfigure}{0.48\columnwidth}
		\includegraphics[width=1.05\linewidth,height = 0.8\linewidth]{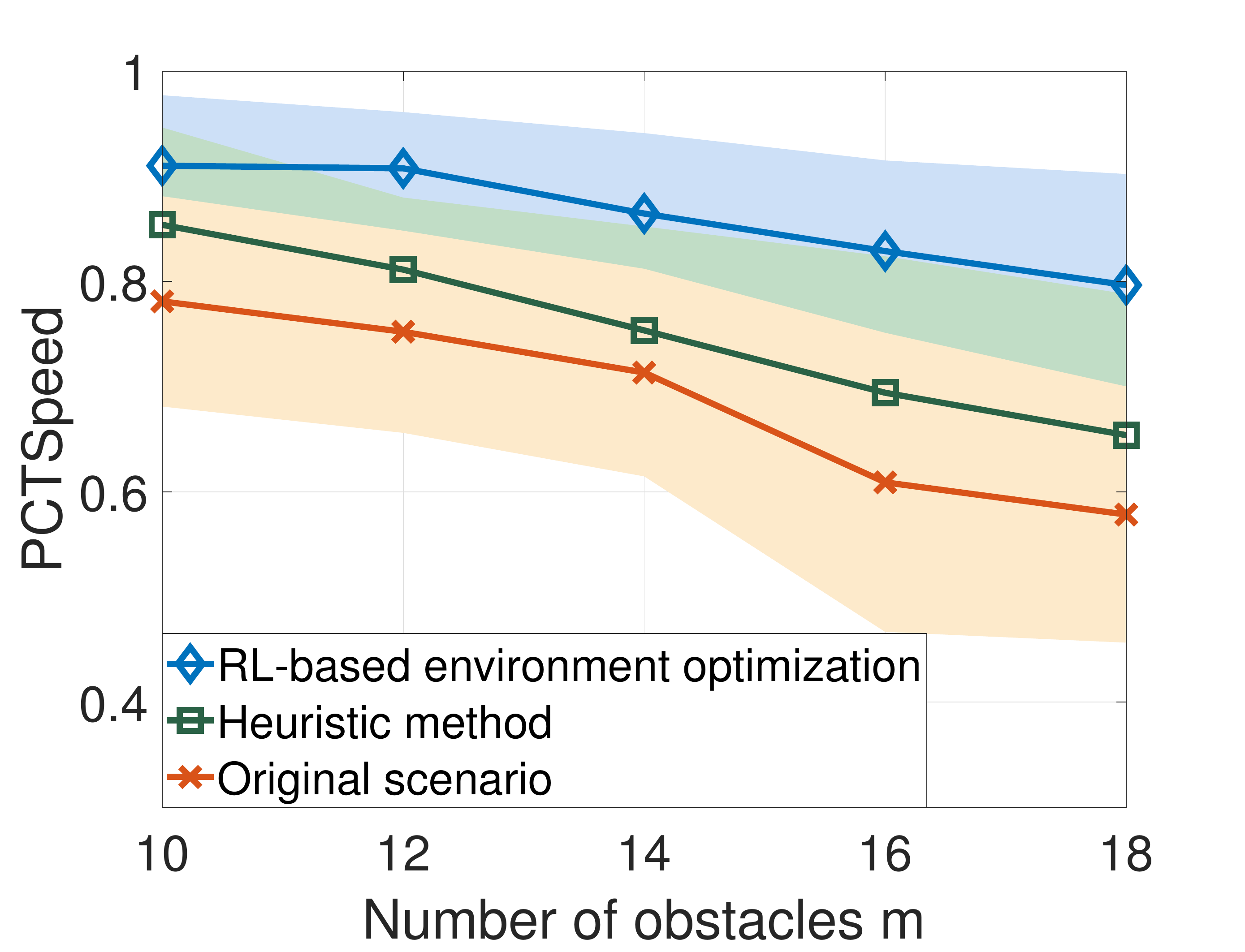}%
		\caption{}%
		\label{subfig_offlinePCT}%
	\end{subfigure}
	\caption{Performance of offline environment optimization compared to the baselines. Results are averaged over $20$ trials and the shaded area shows the std. dev. The RL system is trained on $10$ obstacles and tested on $10$ to $18$ obstacles. (a) SPL ($1$ is best). (b) PCTSpeed.}\label{fig:offline}\vspace{-6mm}
\end{figure}

We consider two navigation scenarios, one in which we perform offline optimization with discrete obstacle motion, and another in which we consider online optimization with continuous obstacle motion. The obstacles have rectangular bodies and the agents have circular bodies. The given agent trajectory planner $\pi_a$ is Reciprocal Velocity Obstacles (RVO) method \cite{van2008reciprocal}.\footnote{The proposed environment optimization approach can be applied in the same manner, irrespective of the chosen agent navigation algorithm.} Two metrics are used: Success weighted by Path Length (SPL) \cite{anderson2018evaluation} and the percentage of the maximal speed (PCTSpeed). The former is a stringent measure combining the success rate and the path length, while the latter is the ratio of the average speed to the maximal one. All results are averaged over $20$ trials with random initial configurations. 

\subsection{Offline Optimization with Discrete Environment}\label{subsec:offline}

We consider a grid environment of size $8 \times 8$ with $10$ obstacles and $4$ agents, which are distributed randomly without overlap. The maximal agent / obstacle velocity is $0.1$ per time step and the maximal time step is $500$.  

\noindent \textbf{Setup.} The environment is modeled as an occupancy grid map. An agent's location is modeled by a one-hot in a matrix of the same dimension. At each step, the policy considers one of the obstacles and moves it to one of the adjacent grid cells. This repeats for $m$ obstacles, referred to as a round, and an episode ends if the maximal round has been reached.

\begin{figure}%
	\centering
	\begin{subfigure}{0.48\columnwidth}
		\includegraphics[width=1.1\linewidth, height = 0.8\linewidth]{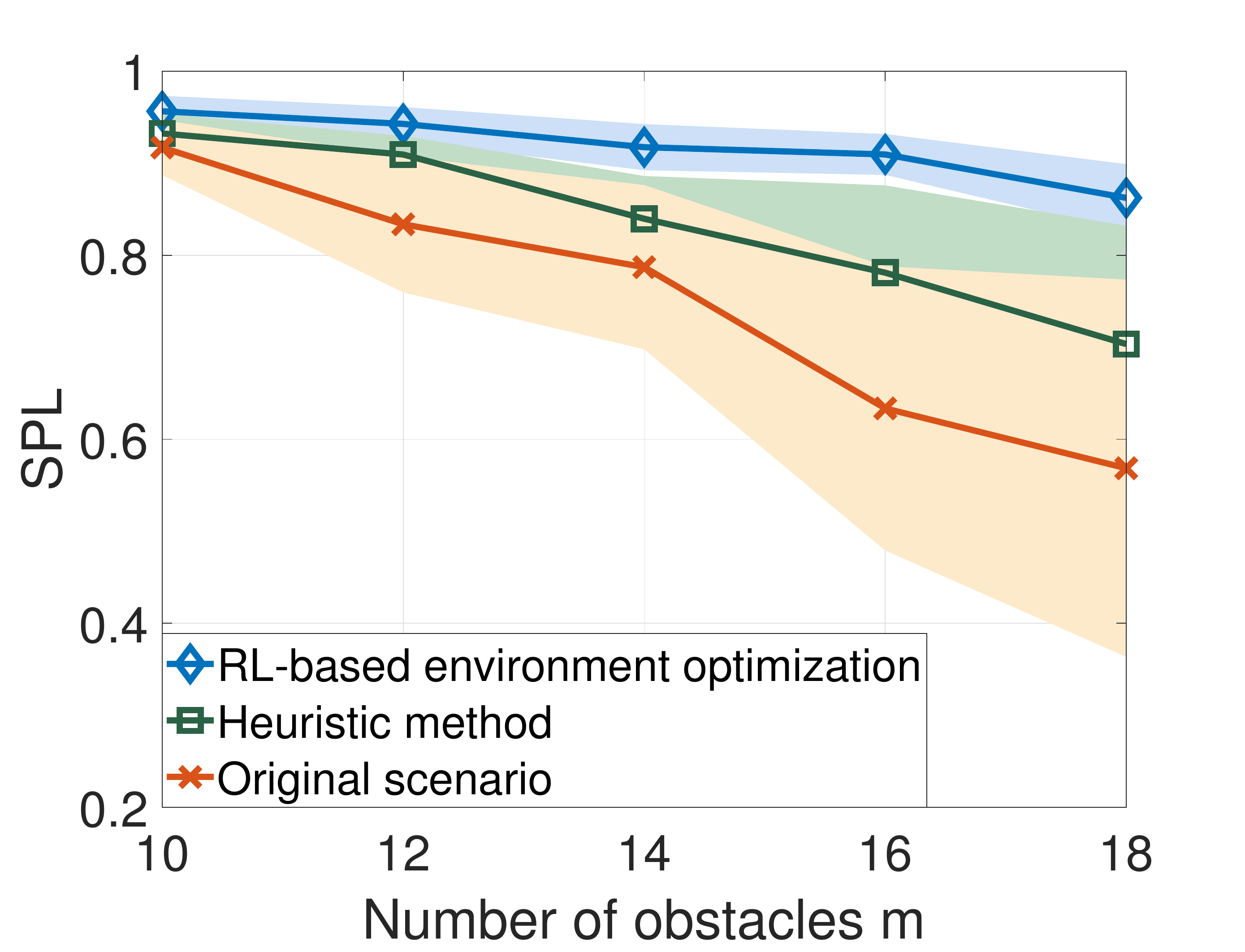}%
		\caption{}%
		\label{subfiga_online_SPL}%
	\end{subfigure}\hfill\hfill%
	\begin{subfigure}{0.48\columnwidth}
		\includegraphics[width=1.05\linewidth,height = 0.8\linewidth]{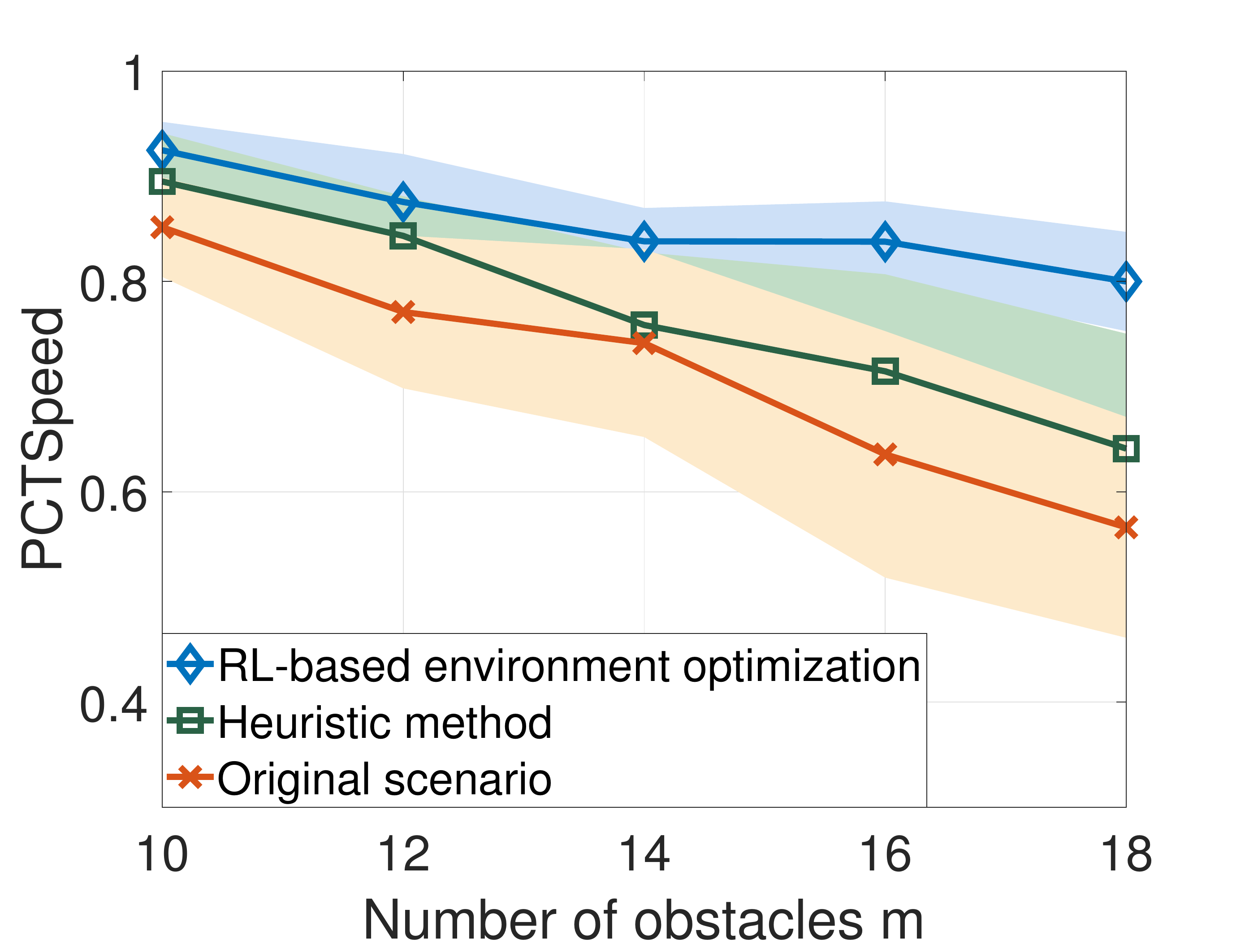}%
		\caption{}%
		\label{subfigb_online_PCTSpeed}%
	\end{subfigure}
	\caption{Performance of online environment optimization compared to the baselines. Results are averaged over $20$ trails and the shaded area shows the std. dev. The RL system is trained on $10$ obstacles and tested on $10$ to $18$ obstacles. (a) SPL ($1$ is best). (b) PCTSpeed.}\label{fig:online}\vspace{-6mm}
\end{figure}

\noindent \textbf{Training.} The objective is to make agents reach destinations quickly while avoiding collision. The team reward is the sum of the PCTSpeed and the ratio of the shortest distance to the traveled distance, while the local reward is the collision penalty of individual obstacle [cf. \eqref{eq:obstacleRewardFunction}]. We parameterize the policy with a CNN of $4$ layers, each containing $25$ features with kernel size $2 \times 2$, and conduct training with PPO \cite{schulman2017proximal}.

\noindent \textbf{Baseline.} Since exhaustive search methods are intractable for our problem, we develop a strong heuristic method to act as a baseline: At each step, one of the obstacles computes the shortest paths of all agents, checks whether it blocks any of these paths, and moves away randomly if so. This repeats for $m$ obstacles, referred to as a round, and the method ends if the maximal round is reached. 

\noindent \textbf{Performance.} We train our model on $10$ obstacles and test on $10$ to $18$ obstacles, which varies obstacle density from $10\%$ to $30\%$. Fig. \ref{fig:offline} shows the results. The proposed approach consistently outperforms baselines with the highest SPL / PCTSpeed and the lowest variance. The heuristic method takes the second place and the original scenario (without any environment modification) performs worst. As we generalize to higher obstacle densities, all methods degrade as expected. However, our approach maintains a satisfactory performance due to the CNN's cabability for generalization. Fig. \ref{fig:offlineDemo} shows an example of how the proposed approach circumvents the dead-locks by optimizing the obstacle layout. Moreover, it improves the path efficiency such that all agents find collision-free trajectories close to their shortest paths.

\subsection{Online Optimization with Continuous Environment}\label{subsec:online}

We proceed to a continuous environment. The agents are initialized randomly in an arbitrarily defined starting region and towards destinations in an arbitrarily defined goal region.

\noindent \textbf{Setup.} The agents and obstacles are modeled by positions $\{\bbp_{a,i}\}_{i=1}^n$, $\{\bbp_{o,j}\}_{j=1}^m$ and velocities $\{\bbv_{a,i}\}_{i=1}^n$, $\{\bbv_{o,j}\}_{j=1}^m$. At each step, each obstacle has a local policy that generates the desired velocity with neighborhood information and we integrate an acceleration-constrained mechanism for position changes. An episode ends if all agents reach destinations or the episode times out. The maximal acceleration is $0.1$, the communication radius is $2$ and the episode time is $500$. 

\begin{figure}%
	\centering
	\begin{subfigure}{0.49\columnwidth}
		\includegraphics[width=1.1\linewidth,height = 0.8\linewidth]{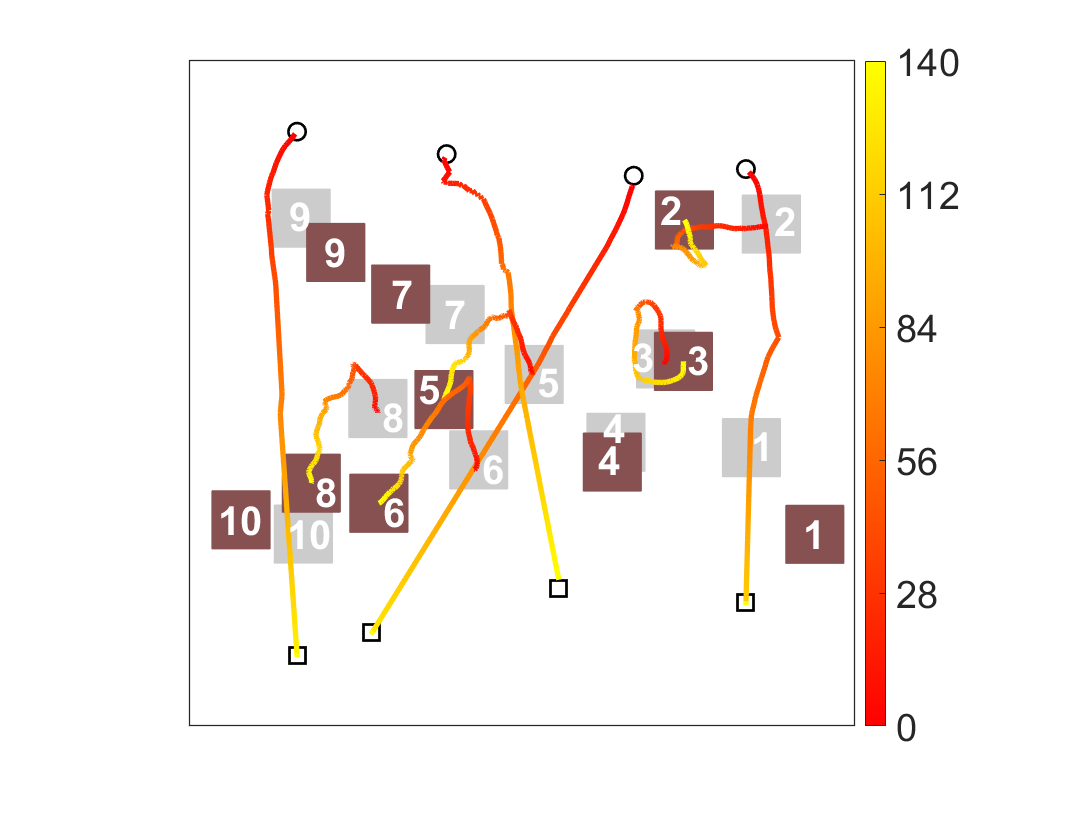}%
		\caption{}%
		\label{subfig_online_demo}%
	\end{subfigure}\hfill\hfill%
	\begin{subfigure}{0.48\columnwidth}
		\includegraphics[width=1.05\linewidth,height = 0.8\linewidth]{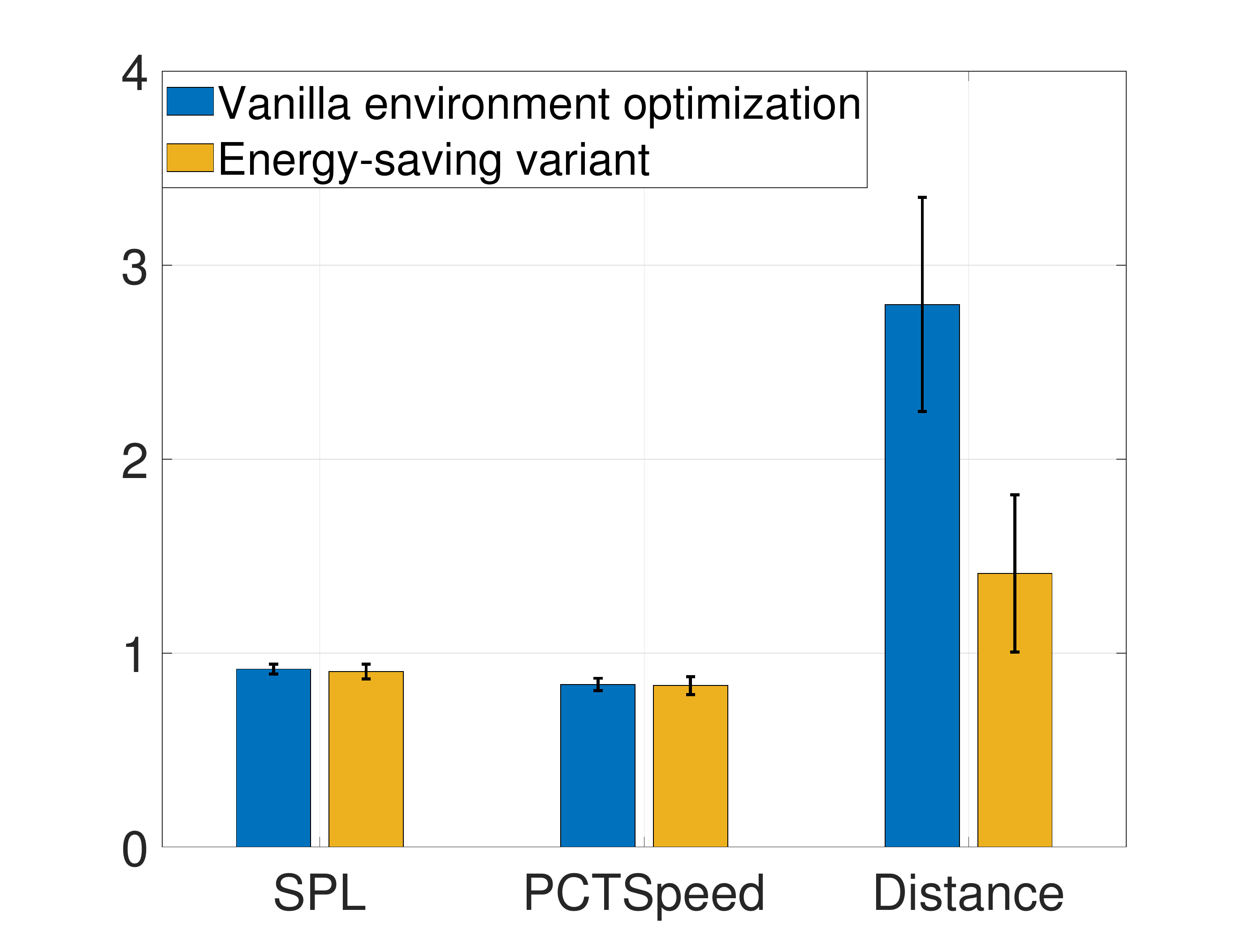}%
		\caption{}%
		\label{subfig_energy_saving}%
	\end{subfigure}%
	\caption{(a) Example of online environment optimization. Circles are the initial positions and squares are the destinations. The grey and brown rectangles are the obstacles before and after environment optimization, and are numbered for exposition. The red-to-yellow lines are trajectories of agents and $5$ example obstacles, and the color bar represents the time scale, showing that no agent-agent nor agent-obstacle collisions occur. (b) Performance of the vanilla environment optimization and the energy-saving variant. The results are averaged over $20$ trials and the error bar shows the std. dev. }\label{fig:onlineDemo}\vspace{-6mm}
\end{figure}

\noindent \textbf{Training.} The team reward in \eqref{eq:obstacleRewardFunction} guides the agents to their destinations as quickly as possible and is defined as \vspace{-2mm}
\begin{align}
	r_{a,i}^{(t)} = \Big(\frac{\bbp_{i}^{(t)}-\bbd_i}{\|\bbp_{i}^{(t)}-\bbd_i\|_2} \cdot \frac{\bbv_{i}^{(t)}}{\|\bbv_{i}^{(t)}\|_2} \Big) \|\bbv_{i}^{(t)}\|_2
\end{align}
at time step $t$, which rewards fast movements towards the destination and penalizes moving away from it. The local reward is the collision penalty. We parameterize the policy with a single-layer GNN. The message aggregation function and feature update function are multi-layer perceptrons (MLPs), and we train the model using PPO.

\noindent \textbf{Performance.} The results are shown in Fig. \ref{fig:online}. The proposed approach exhibits the best performance for both metrics and maintains a good performance for large scenarios. We attribute the latter to the fact that GNNs exploit topological information for feature extraction and are scalable to larger graphs. The heuristic method performs worse but comparably for small number of obstacles, while degrading with the increasing of obstacles. It is note-worthy that the heuristic method is centralized because it requires computing shortest paths of all agents, and hence is not applicable for online optimization and considered here as a benchmark value only for reference. Fig. \ref{subfig_online_demo} shows the moving trajectories of the agents and example obstacles. We see that the obstacles make way for the agents to facilitate navigation such that the agents find trajectories close to their shortest paths.

\noindent \textbf{Optimization criteria.} We show that the objective can capture various criteria. Here, we test whether our approach can learn to also reduce the traveled distance of obstacles. The team reward remains the same, while the local reward becomes the sum of the collision \textit{and speed penalties}. Fig. \ref{subfig_energy_saving} shows the results with $14$ obstacles averaged over $20$ trials. This energy-saving variant achieves a comparable performance (slightly worse) across SPL and PCTSpeed, but saves half of the traveled distance, indicating an effective trade-off between performance and energy expenditure.   

\section{CONCLUSION}

In this paper, we formulated the problem of environment optimization for multi-agent navigation and developed the offline and online solutions. We conducted the completeness analysis for both variants and provided conditions under which all navigation tasks are guaranteed success. We developed a reinforcement learning-based method to solve the problem in a model-free manner, and integrated two different information processing architectures (i.e., CNNs and GNNs) for policy parameterization w.r.t. specific implementation requirements. The developed method is able to generalize to unseen test instances, captures multiple optimization criteria, and allows for a decentralized implementation. Numerical results show superior performance corroborating theoretical findings, i.e., that environment optimization improves the navigation efficiency by optimizing over obstacle regions.


\clearpage





{\small
	\newpage
	\bibliographystyle{IEEEtran}
	\bibliography{myIEEEabrv,biblioOp,AP_bib}
}




\end{document}